\documentclass[12pt]{article}

\usepackage[colorlinks]{hyperref}            
\usepackage{color}
\usepackage{graphicx,subfigure,amsmath,amssymb,amsfonts,bm,epsfig,epsf,url,dsfont}
\usepackage{times}
\usepackage{bbm}      
\usepackage{booktabs}
\usepackage{algorithm,algorithmic}
\usepackage{cases}
\usepackage{fullpage}
\usepackage[small,bf]{caption}
\usepackage{natbib}
\usepackage[top=1in,bottom=1in,left=1in,right=1in]{geometry}

\newtheorem{theorem}{Theorem}
\newtheorem{proposition}{Proposition}
\newtheorem{lemma}{Lemma}
\newtheorem{corollary}{Corollary}
\newcommand{\BlackBox}{\rule{1.5ex}{1.5ex}}  
\newenvironment{proof}{\par\noindent{\bf Proof\ }}{\hfill\BlackBox\\[2mm]}

\def\R{\mathbb{R}}

\def\N{\mathbb{N}}

\def\E{\mathbb{E}}
\def\P{\mathbb{P}}

\def\1{\mathbbm{1}}

\DeclareMathOperator*{\argmax}{\arg\,\max}
\DeclareMathOperator*{\supp}{supp}

\def\X{\mathcal{X}}

\def\uq{\underline{q}}
\def\G{\mathcal{G}}

\def\Rh{\hat{R}}

\def\pb{\bar{\pi}}

\begin{document}

\title{Optimal Discovery with Probabilistic Expert Advice: \\Finite Time Analysis and Macroscopic Optimality}
\author{
S{\'e}bastien Bubeck \\
Department of Operations Research and Financial Engineering, \\
Princeton University \\
{\tt sbubeck@princeton.edu} \\ \\
Damien Ernst \\
Department of Electrical Engineering and Computer Science \\
University of Li\`ege \\
{\tt dernst@ulg.ac.be}  \\ \\
 Aur\'elien Garivier \\
Institut de Math\'ematiques de Toulouse \\
Universit\'e Paul Sabatier \\ 
{\tt aurelien.garivier@math.univ-toulouse.fr}
}

\maketitle
\begin{abstract}%
We consider  an original problem that arises from the issue of security analysis of a power system and that we name optimal discovery with probabilistic expert advice. We address it with an algorithm based on the optimistic paradigm and on the Good-Turing missing mass estimator. We prove two different regret bounds on the performance of this algorithm under weak assumptions on the probabilistic experts. Under more restrictive hypotheses, we also prove a macroscopic optimality result, comparing the algorithm both with  an oracle strategy and with uniform sampling.
Finally, we provide numerical experiments illustrating these theoretical findings.
\end{abstract}

\section{Introduction}
In this paper we consider the following problem: Let $\X$ be a set, and $A \subset \X$ be a set of interesting elements in $\X$. One can access $\X$ only through requests to a finite set of probabilistic experts. More precisely, when one makes a request to the $i^{th}$ expert, the latter draws independently at random a point from a fixed probability distribution $P_i$ over $\X$. One is interested in discovering rapidly as many elements of $A$ as possible, by making sequential requests to the experts.

\subsection{Motivation}
The original motivation for this problem arises from the issue of real-time security analysis of a power system. This problem often amounts to identifying in a  set of  ‘credible’ contingencies those that may indeed endanger the security of the power system and perhaps lead to a  system collapse with catastrophic consequences (e.g., an entire region, country may be without electrical power for hours). Once those dangerous contingencies have been identified,   the system operators usually take preventive actions so as to ensure that they could mitigate their effect on the system in the likelihood they would occur. Note that usually, the dangerous contingencies are very rare with respect to the non dangerous ones.
A straightforward approach for tackling this security analysis problem is to simulate the power system dynamics for every credible contingency so as to identify those that are indeed dangerous.
Unfortunately, when the set of credible contingencies contains a large number of elements (say, there are more than $10^5$ ‘credible’ contingencies) such an approach may not possible anymore since the computational resources required to simulate every contingency may excess those that are usually available during the few (tens of) minutes available for the real-time security analysis. One is therefore left with the problem of identifying within this short time-frame a maximum number of dangerous contingencies rather than all of them.
The approach proposed in \citet{Fonteneau11phd} and \citet{FonteneauErnstDruetPanciaticiWehenkel10power} addresses this problem by building first very rapidly what could be described as a probability distribution $P$ over the set of credible contingencies that points with  significant probability to contingencies which are dangerous. Afterwards, this probability distribution is used to draw the contingencies to be analyzed through simulations. When the computational resources are exhausted, the approach outputs the contingencies found to be dangerous.
One of the main shortcoming of this approach is that usually $P$ points only with a significant probability to a few of the dangerous contingencies and not all of them. This in turn makes this probability distribution not more likely to generate after a few draws new dangerous contingencies than for example a uniform one. The dangerous contingencies to which $P$ points to with a significant probability depend however strongly on the set of (sometimes arbitrary) engineering choices that have been made for building it. One possible strategy to ensure that more dangerous contingencies can be identified within a limited budget of draws would therefore be to consider $K>1$ sets of engineering choices  to build $K$ different probability distributions $P_1$, $P_2$, $\ldots$, $P_K$ and to draw the contingencies from these $K$ distributions rather than only from a single one. This strategy raises however an important question to which this paper tries to answer: how should the distributions be selected  for being 
able to generate with a given number of draws a maximum number of dangerous contingencies?
We consider the specific case where the contingencies are sequentially drawn and where the distribution selected for generating a contingency at one instant can be based on the past distributions that have been selected, the contingencies that have been already drawn and the results of the security analyses (dangerous/non dangerous) for these contingencies.
This corresponds exactly to the optimal discovery problem with expert advice described above.
We believe that this framework has many other possible applications, such as for example web-based content access.

\subsection{Setting and Notation} \label{sec:settingnotation}
In this paper we restrict our attention to finite or countably infinite  sets $\X$. 
We denote by $K$ the number of experts. For each $i\in\{1,\dots,K\}$, we assume that $(X_{i,n})_{n\geq 1}$ are random variables with distribution $P_i$ such that the $(X_{i,n})_{i,n}$ are independent. Sequential discovery with probabilistic expert advice can be described as follows: at each time step $t\in\N^*$, one picks an index $I_t\in\{1,\dots, K\}$, and one observes $X_{I_t, n_{I_t,t}}$, where 
\[n_{i,t} = \sum_{s\leq t} \1{\{I_s = i\}}\;.\]
The goal is to choose the $(I_t)_{t\geq 1}$ so as to observe as many elements of $A$ as possible in a fixed horizon $t$, that is to maximize the number of interesting items found after $t$ requests
\begin{equation} \label{eq:numberfound}
F(t) = \sum_{x \in A} \1\bigg\{x \in \{X_{1,1},\hdots,X_{1,n_{1,t}}, \hdots, X_{K,1},\hdots,X_{K,n_{K,t}}\}\bigg\}.
\end{equation} 
Note in particular that it is of no interest to observe twice the same same element of $A$. The index $I_{t+1}$ may be chosen according to past observations: it is a (possibly randomized) function of $(I_1, X_{I_1,1}, \dots, I_{t}, X_{I_t, n_{I_t,t}})$. 

An easier quantity to analyze than the number of interesting items found $F(t)$ is the waiting time $T(\lambda)$, $\lambda \in (0,1)$, which is the time at which the strategy has a missing mass of interesting items smaller than $\lambda$ on every experts, that is
\begin{equation} \label{eq:waitingtime}
T(\lambda) = \inf \bigg\{t : \forall i \in \{1, \hdots, K\}, P_i(A \setminus \{X_{1,1},\hdots,X_{1,n_{1,t}}, \hdots, X_{K,1},\hdots,X_{K,n_{K,t}}\}) \leq \lambda \bigg\}.
\end{equation}

While we shall derive a general strategy that can be used without any assumption on the probabilistic experts, for the mathematical analysis of the waiting time $T(\lambda)$ we make the following assumption:
\begin{enumerate}
\item[(i)] non-intersecting supports: $A \cap \supp(P_i) \cap \supp(P_j) = \emptyset$ for $i \neq j$.
\end{enumerate}
Furthermore we will also derive some results under the following more restrictive assumptions:
\begin{enumerate}
\item[(ii)] finite supports with the same cardinality: $|\supp(P_i)| = N, \forall i \in \{1,\hdots,K\}$,
\item[(iii)] uniform distributions: $P_i(x) = \frac{1}{N}, \forall x \in \supp(P_i), \forall i \in \{1,\hdots,K\}$.
\end{enumerate}

\subsection{Contribution and Content of the Paper}
This paper contains the description of a generic algorithm for the optimal discovery problem with probabilistic expert advice, and 
a theoretical analysis of its properties.
In Section~\ref{sec:goodUCB}, we first depict our strategy, termed Good-UCB. This algorithm relies on the \emph{optimistic paradigm}, which led to the UCB (Upper Confidence Bound) algorithm for multi-armed bandits, see \citet{AuerEtAl02FiniteTime} and \citet{GarivierCappe11KLUCB}. It relies also on a finite-time analysis of the Good-Turing estimator for the missing mass. We also derive in Section~\ref{sec:goodUCB} two different regret bounds under the non-intersecting assumption (i): we first show that $F^{UCB}(t)$ (the number of interesting items found by Good-UCB) is larger than $F^*(t)$ (the number of interesting items found by an oracle strategy), up to a term of order $\sqrt{K t \log(t)}$. We argue that such a bound does not capture all the fine properties of Good-UCB: indeed, on the contrary to the multi-armed bandit problem, here the regret $F^*(t) - F(t)$ remains bounded for any reasonable strategy. This can be understood as a \emph{restoring property} of the game: if a policy makes a sub-optimal choice at 
some given time $t$, then in the future it will have better opportunities than the optimal policy. This key feature of our problem prevents the regret from growing too much.
To analyze this phenomenon, we complete our first bound by a second regret analysis---the main result of the paper---which states roughly that with high probability, $T_{UCB}(\lambda)$ (the waiting time for the strategy Good-UCB) is \emph{uniformly} (in $\lambda$) smaller than $T^*(\lambda')$ (the smallest possible waiting time), for some $\lambda'$ close to $\lambda$ and up to a small additional term, see Theorem \ref{th:nlregret} for a more precise statement. We emphasize that these regret bounds are both completely distribution-free and explicit. 

In Section~\ref{sec:macro} we propose to investigate the behavior of Good-UCB in a {\em macroscopic limit} sense, that is we make assumptions [(i), (ii), (iii)] and we consider the limit when the size of the set $\X$ grows to infinity while maintaining a constant proportion of interesting items. In this scenario we show that Good-UCB is macroscopically optimal, in the sense that the normalized waiting time of Good-UCB tends to the normalized smallest possible waiting time. We also derive a formula for this latter quantity and we show that it is equal to $\sum_{i: q_i>\lambda} \log\frac{q_i}{\lambda}$, where $q_i$ is the limiting proportion of interesting items on expert $i$. This macroscopic limit also allows to easily assess the performance of different strategies, and we show that for example the normalized waiting time of uniform sampling tends to $K \max_{1 \leq i \leq K} \log\frac{q_i}{\lambda}$, which proves that this strategy is macroscopically suboptimal, unless all experts have the same number of 
interesting items.

Finally, Section~\ref{sec:simus} reports experimental results that show that
the Good-UCB  algorithm performs very  well, even in a setting where
assumptions (i), (ii) and  (iii) are not satisfied.
The appendix contains some technical proofs, together with a more detailed discussion on oracle strategies in the macroscopic limit and on the relation between the waiting time $T$ defined in \eqref{eq:waitingtime} and the number of items found $F$ defined in \eqref{eq:numberfound}, proving in particular that optimality in terms of waiting time is equivalent to optimality in terms of number of items found.

\section{The Good-UCB Algorithm}~\label{sec:goodUCB}
We describe here the Good-UCB strategy. This algorithm is a sequential method estimating at time $t$, for each expert $i\in\{1,\dots,K\}$, the total probability of the interesting items that remain to be discovered through requests to expert $i$. This estimation is done by adapting the so-called Good-Turing estimator for the missing mass. Then, instead of simply using the distribution with highest estimated missing mass, which proves hazardous, we make use of the \emph{optimistic paradigm}---see  \citet[Chapter~2, and references therein]{BC12}---a heuristic principle well-known in reinforcement learning, which entails to prefer using an \emph{upper-confidence bound} (UCB) of the missing mass instead. At a given time step, the Good-UCB algorithm simply makes a request to the expert with highest upper-confidence bound on the missing mass at this time step.

We start in Section \ref{sec:est} with the Good-Turing estimator and a brief study of its concentration properties. Then we describe precisely the Good-UCB strategy in Section \ref{sec:GoodTuringEstimator}. Next we proceed to the theoretical analysis of Good-UCB and we start in Section \ref{sec:oclnew} where we describe an oracle strategy (that we shall use as a comparator) that we prove to be optimal under assumption (i). In Section \ref{sec:classical} we show that one can obtain a standard regret bound of order $\sqrt{t}$ when one compares the number of items $F^{UCB}(t)$ found by Good-UCB to the number of items $F^*(t)$ found by the oracle. This bound is not completely satisfactory (as we explain in Section \ref{sec:classical}), and our main result---a 'non-linear' regret bound---is proved in Section \ref{sec:nonlinear}.

\subsection{Estimating the Missing Mass} \label{sec:est}
Our algorithm relies on an estimation at each step of the probability of obtaining a new interesting item by making a request to a given expert. A similar issue was addressed by I.~Good and A.~Turing  as part of their efforts to crack German ciphers for the Enigma machine during World War II. In this subsection, we describe a version of the Good-Turing estimator adapted to our problem.
Let $\Omega$ be a discrete set, and let $A$ be a subset of interesting elements of $\Omega$.
Assume that $X_1,\dots,X_n$ are elements of $\Omega$ drawn independently under the same distribution $P$, and define for every $x\in\Omega$:
\[O_n(x) = \sum_{m=1}^n \1\{X_m=x\},\quad Z_n(x) = \1\{O_n(x) = 0\},\quad U_n(x) = \1\{O_n(x) = 1\}\;. \]
Let $R_n = \sum_{x\in A} Z_n(x)P(x)$ denote the missing mass of the interesting items, and let $U_n = \sum_{x\in A} U_n(x)$ be the number of elements of $A$ that have been seen exactly once (in linguistics, they are often called \emph{hapaxes}).
The idea of the Good-Turing estimator---see \citet{Good53GT}, see also \citet[and references therein]{McallesterSchapire00GoodTuring,orlitsky2003always}---is to estimate the (random) ``missing mass'' $R_n$, which is the total probability of all the interesting items that do not occur in the sample $X_1,\dots,X_n$, by the ``fraction of hapaxes $\Rh_n = U_n/n$.
This estimator is well-known in linguistics, for instance in order to estimate the number of words in some language, see~\citet{Gale95GT}.
We shall use the following tight bound on the estimation error. We emphasize the fact that the following bound holds true \emph{independently of the underlying distribution $P$}.
\begin{proposition}\label{prop:inegGT}
With probability at least $1-\delta$,
 \[ \Rh_n -\frac{1}{n} - (1+\sqrt{2}) \sqrt{\frac{\log(4/\delta)}{n}} \leq R_n  \leq  \Rh_n+ (1+\sqrt{2}) \sqrt{\frac{\log(4/\delta)}{n}}\;. \]
\end{proposition}

\begin{proof}
For self-containment, we first show that $\E R_n - \E \Rh_n \in \left[-\frac{1}{n},0\right]$; this result is well known, see for example Theorem 1 in \citet{McallesterSchapire00GoodTuring}:
\begin{align*}
\E R_n - \E \Rh_n  &= \sum_{x\in A} \left[P(x)\left(1-P(x)\right)^n - \frac{1}{n}\times nP(x)\left(1-P(x)\right)^{n-1}\right]\\
 & = -\frac{1}{n} \sum_{x\in A} P(x) \times nP(x)\left(1-P(x)\right)^{n-1} \\
 & = -\frac{1}{n}\E\left[ \sum_{x\in A}P(x) U_n(x)\right]\in \left[-\frac{1}{n},0\right]\;.
\end{align*}
Next we apply the inequality of \cite{Mcdiarmid89boundedDiff} to $\Rh_n$ as follows. The random variable $\Rh_n$ is a function of the independent observations $X_1,\dots,X_n$ such that, denoting $\Rh_n = f(X_1,\dots,X_n)$, modifying just one observation has limited impact: $\forall l\in\{1,\dots,n\}, \forall (x_1,\dots,x_n,x'_l)\in\Omega^{n+1}$,
\[
\left|f(x_1,\dots,x_n) - f(x_1,\dots,x_{l-1},x'_l,x_{l+1}, \dots,x_n)\right| \leq \frac{2}{n}.
\]
Thus one gets that, with probability at least $1-\delta$, 
\[\left|\Rh_n - \E[\Rh_n]\right| \leq \sqrt{\frac{2 \log(2/\delta)}{n}}\;.\]
Finally we extract the following result from Theorem 10 and Theorem 16 in~\citet{McAllesterOrtiz03concentration}: with probability at least $1-\delta$,
\[\left|R_n - \E[R_n]\right| \leq \sqrt{\frac{\log(2/\delta)}{n}}\;.\]
which concludes the proof.
\end{proof}

\subsection{The Good-UCB Algorithm}\label{sec:GoodTuringEstimator}
Following the example of the well-known Upper-Confidence Bound procedure for multi-armed bandit problems, we propose Algorithm~\ref{algo:GUCB}, which we call \emph{Good-UCB} in reference to the estimator it relies on.
For each arm $i\in\{1,\dots,K\}$, the index at time $t$ of Good-UCB corresponds to the estimate 
\[\Rh_{i, n_{i,t-1}} = \frac{1}{n_{i,t-1}} \sum_{x \in A} \1\left\{\sum_{s=1}^{n_{i,t-1}} \1\{X_{i,s} = x\} = 1 \; \text{and} \; \sum_{j=1}^K \sum_{s=1}^{n_{j,t-1}} \1\{X_{j,s} = x\} =1 \right\}\]
 of the missing mass 
\begin{equation} \label{eq:seb1}
\sum_{x\in A \setminus \left\{X_{I_{1}, n_{I_1, 1}}, \dots, X_{I_{t-1}, n_{I_{t-1}, t-1}}\right\}} P_i(x)
\end{equation}
 inflated by a confidence bonus of order $\sqrt{\log(t)/n_{i,t-1}}$.
Good-UCB relies on a tuning parameter $C$ which is discussed below.

\begin{algorithm}
\caption{Good-UCB}
\label{algo:GUCB}
\begin{algorithmic}[1]
\STATE For $1\leq t\leq K$ choose $I_t = t$.
\FOR {$t \geq K+1$}
\STATE Choose $I_t =\argmax_{1\leq i\leq K} \left\{\Rh_{i_,n_{i,t-1}} + C \sqrt{\frac{\log{(4 t)}}{n_{i,t-1}}} \right\}$
\STATE Observe $X_t$ distributed as $P_{I_t}$ and update the missing mass estimates accordingly
\ENDFOR
\end{algorithmic}
\end{algorithm}

The Good-UCB algorithm is designed to work without any assumption on the probabilistic experts. However for the analysis we shall make the non-intersecting supports assumption (i). Indeed without this assumption the missing mass of a given expert $i$ depends explicitly on the outcomes of {\em all} requests (and not only requests to expert $i$), see \eqref{eq:seb1}, which makes the analysis significantly more difficult. On the other hand under assumption (i) one can define the missing mass of expert $i$ after $n$ pulls without any reference to the other arms, and it takes the following simple form:
\begin{equation} \label{eq:seb2}
R_{i,n} = \sum_{x\in A \setminus \{X_{i,1}, \dots, X_{i, n}\}} P_i(x)\;.
\end{equation}
Note that while the theoretical analysis will be carried out under assumption (i), we show in Section~\ref{sec:simus} that Good-UCB performs well in practice even when this assumption is not met.

\subsection{The Closed-loop Oracle Policy} \label{sec:oclnew}
In this section we define a policy that we shall use as a benchmark to study the properties of Good-UCB. We assume hereafter that assumption (i) is satisfied (in particular we shall use the notation defined in \eqref{eq:seb2}). The Oracle Closed-Loop policy, denoted OCL in the following, makes a request at time $t$ to the expert
\[I^*_t = \argmax_{1\leq i\leq K} R_{i, n^*_{i, t-1}}\;, \;\; \hbox{where} \;\; n^*_{i,t} = \sum_{s=1}^t \1\{I^*_s=i\}\;.\]
In words, OCL greedily selects the expert that maximizes the probability of finding a new interesting item. The next lemma shows that this greedy procedure is in fact optimal (in expectation) under assumption (i). The proof is given in the appendix. 

For any given policy $\pi$, let $F^{\pi}(t)$ be the number of items found at time $t$ with $\pi$, $I_t^{\pi}$ be the expert chosen by $\pi$ at time $t$, and $n_{i,t}^{\pi} = \sum_{s=1}^t \1\{I_s^{\pi}=i\}$ be the number of requests made by $\pi$ to expert $i$ up to time $t$. 

\begin{lemma}\label{lem:OCLoptimal}
Let $\pi$ be an arbitrary policy, and $t \geq 1$. Then
\[
\E F^{\pi}(t) \leq \E F^*(t) .
\]
\end{lemma}

The optimality of OCL crucially relies on assumption (i). Consider for example the following problem instance: $\X =\{1,2,3,4\}$, $A=\{1,2,3\}$, $K=3$, $\nu_1 = \delta_1$, $\nu_2 = \frac{2}{5} (\delta_1 + \delta_2) + \frac{1}{5}\delta_4$, and $\nu_3 = \frac{2}{5} (\delta_1 + \delta_3) + \frac{1}{5}\delta_4$ and $t=2$. 
In this case OCL first chooses expert 1, and then (say) expert 2: this yields $F^*(2)=1+2/5=7/5$. But the strategy $\pi$ consisting in choosing first expert 2, and then expert 3, is readily seen to have expected return $\E F^{\pi}(2) = 2/5\times(1+ 2/5) + 2/5\times(1+4/5) + 1/5 \times 4/5 = 36/25>7/5$.

The next lemma is a technical result on OCL that shall prove to be very useful to derive a standard regret bound for Good-UCB. Its proof is also given in the appendix. 

\begin{lemma}\label{lem:firstBound}
Let $\pi$ be an arbitrary policy, and for $t \geq 1$ let 
\[
\bar{I}_t = \argmax_{1\leq i\leq K}R_{i, n^{\pi}_{i,t-1}}\;.\]
Then
\[
\E F^*(t) \leq \sum_{s=1}^t \E R_{\bar{I}_s, n^{\pi}_{\bar{I}_s,s-1}} .
\]
\end{lemma}

\subsection{Classical Analysis of the Good-UCB Algorithm} \label{sec:classical}
We provide here an upper bound on the expectation of $F^*(t)-F^{UCB}(t)$ which is completely \linebreak[4] distribution-free, and which depends only on the horizon $t$ and on the number $K$ of experts. This bound grows like $O(\sqrt{Kt\log(t)})$, which is a usual rate for a bandit problem. Indeed,  thanks to Lemma~\ref{lem:firstBound}, the analysis presented in this section follows the lines of classical regret analyses, see for instance~\citet[and the references therein]{BC12}.
Below, we discuss some differences between the discovery problem considered here and bandit problems, and we provide an alternative analysis of the Good-UCB algorithm which is more suited to understand its long-term behavior. 

\begin{theorem}\label{th:firstBound}
For any $t \geq 1$, under assumption (i), Good-UCB (with constant $C=(1+\sqrt{2})\sqrt{3}$) satisfies
\[\E\left[F^*(t)-F^{UCB}(t)\right] \leq 17 \sqrt{K t \log(t)}  + 20\sqrt{K t} + K+ K\log(t/K) \;.\]
\end{theorem}

\begin{proof}
Consider the event
\begin{multline*}
 \xi = \bigg\{\forall i \in \{1, \hdots, K\}, \forall u > \sqrt{K t}, \forall s\leq u,\\ \Rh_{i,s} -\frac{1}{s} - (1+\sqrt{2}) \sqrt{\frac{3 \log(4 u)}{s}} \leq R_{i,s}  \leq  \Rh_{i,s} + (1+\sqrt{2}) \sqrt{\frac{3 \log(4 u)}{s}}\bigg\} \;.
\end{multline*}
Using Proposition \ref{prop:inegGT} and an union bound, one obtains $\P(\xi) \geq 1 - \sqrt{\frac{K}{t}}$, and thus
\begin{equation*} \E\left[(F^*(t)-F^{UCB}(t)) (1-\1_{\xi}) \right] \leq t\sqrt{\frac{K}{t}} = \sqrt{Kt}\;.\label{eq:surAbarre}\end{equation*}
Let $u>\sqrt{Kt}$ and $\bar{I}_u = \argmax_{1\leq i\leq K}R_{i, n_{i,u-1}}$ be defined as in Lemma \ref{lem:firstBound}. On the event $\xi$, one obtains by definition of $I_u$ that
\begin{align*}
 R_{I_u, n_{I_u, u-1}} &\geq \Rh_{I_u, n_{I_u, u-1}} - \frac{1}{n_{I_u, u-1}} - (1+\sqrt{2}) \sqrt{\frac{3 \log(4 u)}{n_{I_u, u-1}}}\\
 &\geq \Rh_{I_u, n_{I_u, u-1}} + (1+\sqrt{2}) \sqrt{\frac{3 \log(4 u)}{n_{I_u, u-1}}}   - \frac{1}{n_{I_u, u-1}} - 2(1+\sqrt{2}) \sqrt{\frac{3 \log(4 u)}{n_{I_u, u-1}}}\\
 &\geq \Rh_{\bar{I}_u, n_{\bar{I}_u, u-1}} + (1+\sqrt{2}) \sqrt{\frac{3 \log(4 u)}{n_{I^*_u, u-1}}}   - \frac{1}{n_{I_u, u-1}} - 2(1+\sqrt{2}) \sqrt{\frac{3 \log(4 u)}{n_{I_u, u-1}}}\\
 &\geq R_{\bar{I}_u, n_{\bar{I}_u, u-1}}  - \frac{1}{n_{I_u, u-1}} - 2(1+\sqrt{2}) \sqrt{\frac{3 \log(4 u)}{n_{I_u, u-1}}}\;,\\
\end{align*}
and thus
\begin{align*}
R_{\bar{I}_u, n_{\bar{I}_u,u-1}} - R_{I_u, n_{I_u,u-1}}  & \leq  \frac{1}{n_{I_u, u-1}} + 2(1+\sqrt{2}) \sqrt{\frac{3 \log(4 u)}{n_{I_u, u-1}}}\\
& \leq   \frac{1}{n_{I_u, u-1}} + 2(1+\sqrt{2}) \sqrt{\frac{3 \log(4 t)}{n_{I_u, u-1}}}\;.
\end{align*}
Hence, using Lemma~\ref{lem:firstBound} and the above computation, one obtains
\begin{align}
 \E\left[F^*(t)-F^{UCB}(t)\right] &\leq \sqrt{Kt} + \E\left[\sum_{u=1}^t \frac{1}{n_{I_u, u-1}} + 2(1+\sqrt{2}) \sqrt{\frac{3 \log(4 t)}{n_{I_u, u-1}}} \right]\nonumber\\
& = \sqrt{Kt} +\E\left[\sum_{i=1}^K \sum_{s=1}^{n_{i,  t-1}} \frac{1}{s}+ 2(1+\sqrt{2}) \sqrt{\frac{3 \log(4 t)}{s}}\right]\nonumber\\
& \leq  \sqrt{Kt} + \E\left[\sum_{i=1}^K 1+ \log(n_{i,  t-1}) + 4(1+\sqrt{2}) \sqrt{3\log(4t)(n_{i, t-1}+1)} \right] \nonumber\\
&\leq \sqrt{Kt} + K + K\log(t/K)  +  4(1+\sqrt{2}) \sqrt{3 K t\log(4t)} \nonumber
\end{align}
by Jensen's inequality and the fact that $\sum_{i=1}^K n_{i, t-1} = t-1$.
\end{proof}
 
The cumulative regret bound provided in Theorem \ref{th:firstBound} has a similar flavor as well known regret bounds for the multi-armed bandit problem. Unfortunately here, such bounds, by suggesting that the regret increases with $t$, do not represent completely the behavior of Good-UCB: as we shall see in the experiments, the difference between $F^*(t)$ and $F^{UCB}(t)$ is bounded and tends to $0$ as $t$ tends to infinity (indeed, ultimately any reasonable strategy will find all the interesting items). Theorem~\ref{th:firstBound} provides insight into the properties of Good-UCB only for 'small' values of $t$.

The weakness of Theorem~\ref{th:firstBound} and its analysis is that, by using the upper bound of Lemma~\ref{lem:firstBound}, one ignores the \emph{restoring property} of the game: if a policy makes a sub-optimal choice at some given time $t$, then it will have better opportunities than OCL in the future, which prevents the regret from growing too much. In the next section we provide a completely different analysis of Good-UCB that takes advantage of this restoring property. This results in a non-standard regret bound, which differs from usual results in the multi-armed bandit literature.

Let us make one more comment about the bound of Theorem~\ref{th:firstBound}. On the contrary to the multi-armed bandit, the discovery problem discussed in this paper has a 'natural' time scale: if the horizon $t$ is too small, then even OCL will not be able to discover a significant proportion of interesting items, while if $t$ is too large then any reasonable strategy will find almost all interesting items. To go around this issue we find it more elegant to study the waiting time $T(\lambda)$ (see \eqref{eq:waitingtime}) which yields a sort of automatic normalization of the time scale.

\subsection{Time-uniform Analysis of the Good-UCB Algorithm} \label{sec:nonlinear}
In this section we analyze the waiting time of Good-UCB under assumption (i). We shall derive a non-linear regret bound as follows. For a fixed $\lambda \in (0,1)$ we consider the number of requests $T_{UCB}(\lambda)$ that Good-UCB needs to make in order to have a missing mass of interesting items smaller than $\lambda$ on each expert, see \eqref{eq:waitingtime}. We also consider the omniscient oracle strategy that minimizes this number of requests, given the knowledge of $\lambda$ and the sequence of answers to the requests $(X_{i,s})_{1\leq i \leq K, s \geq 1}$. We denote by $T^*(\lambda)$ the corresponding number of requests for this omniscient oracle strategy. (Note that this strategy is even more powerful than the OCL studied in the previous sections.) We now prove that with high probability, $T_{UCB}(\lambda)$ is smaller than $T^*(\lambda')$, for some $\lambda'$ close to $\lambda$ and up to a small additional term. 

\begin{theorem} \label{th:nlregret}
Let $c > 0$ and $S \geq 1$. Under assumption (i), Good-UCB (with constant $C=(1+\sqrt{2})\sqrt{c+2}$) satisfies with probability at least $1 - \frac{K}{c S^{c}}$, for any $\lambda \in (0,1)$,
\begin{multline*}T_{UCB}(\lambda) \leq T^* + K S \log \left(8 T^* + 16 K S \log( K S ) \right),\\ \text{where} \quad T^* = T^*\left(\lambda - \frac{3}{S} - 2  (1+\sqrt{2}) \sqrt{\frac{c+2}{S}} \right)\;.\end{multline*}
\end{theorem}

Informally this bound shows that Good-UCB slightly lags behind the omniscient oracle strategy. Under more restrictive assumptions on the experts it is possible to obtain a more explicit bound by studying the variations of $T$. In the next section we take another route and we show that the above upper bound can be used to prove a clear qualitative property for Good-UCB, namely its {\em macroscopic optimality}.

\begin{proof}
Recall that we work under assumption (i), and we run Good-UCB  with parameter $C=(1+\sqrt{2})\sqrt{c+2}$, for some positive constant $c$. After $t$ pulls, the missing mass estimate of expert $i$ is:
\[\Rh_{i,t} = \frac{1}{t} \sum_{x \in A} \1\left\{1 = \sum_{s=1}^{t} \1\{X_{i,s} = x\} \right\} .
\] 
We consider the following event:
\begin{multline*}
 \xi = \bigg\{\forall i \in \{1, \hdots, K\}, \forall t > S, \forall s\leq t,\\ \Rh_{i,s} -\frac{1}{s} - (1+\sqrt{2}) \sqrt{\frac{(c+2) \log(4 t)}{s}} \leq R_{i,s}  \leq  \Rh_{i,s} + (1+\sqrt{2}) \sqrt{\frac{(c+2) \log(4 t)}{s}}\bigg\} \;.
\end{multline*}
Using Proposition \ref{prop:inegGT} and an union bound, one obtains $\P(\xi) \geq 1 - \frac{K}{c S^{c}}$. In the following we work on the event $\xi$. Recall that $T^*(\lambda)$ (respectively ${T}_{UCB}(\lambda)$) is the time at which the omniscient oracle strategy (respectively the Good-UCB strategy) attains a missing mass smaller than $\lambda$ on all experts. Note that $T^*(\lambda)$ and ${T}_{UCB}(\lambda)$ are functions of $(X_{i,s})_{1\leq i \leq K, s \geq 1}$. In particular one can write:
\begin{align*}
& {T}_{UCB}(\lambda) = \min \left\{t \geq 1 : \forall i \in \{1, \hdots, K\}, R_{i, n_{i,t}} \leq \lambda \right\} , \\
& T^*(\lambda) = \sum_{i=1}^K T_i^*(\lambda), \;\; \text{where} \;\; T_i^*(\lambda)  = \min \left\{t \geq 1 : R_{i, t} \leq \lambda \right\} .
\end{align*}
Let
$$U(\lambda) = \min \left\{t \geq 1 : \forall i \in \{1, \hdots, K\}, \Rh_{i,n_{i,t}} + (1+\sqrt{2}) \sqrt{\frac{(c+2) \log(4 t)}{n_{i,t}}} \leq \lambda \right\} .$$
Let $S' \geq S$ to be defined later. On the event $\xi$ one clearly gets 
${T}_{UCB}(\lambda) \leq \max(S', U(\lambda))$. 
Moreover the following inequalities hold true if $U(\lambda) > S'$ (see below for an explanation of each inequality)
\begin{align*}
R_{i, n_{i,U(\lambda)}}  &\geq  \Rh_{i,n_{i,U(\lambda)}} -\frac{1}{n_{i,U(\lambda)}} - (1+\sqrt{2}) \sqrt{\frac{(c+2) \log(4 U(\lambda))}{n_{i,U(\lambda)}}} \\
& \geq  \Rh_{i,n_{i,U(\lambda)}-1} -\frac{3}{n_{i,U(\lambda)}} - (1+\sqrt{2}) \sqrt{\frac{(c+2) \log(4 U(\lambda))}{n_{i,U(\lambda)}}} \\
& \geq  \left(\lambda - (1+\sqrt{2}) \sqrt{\frac{(c+2) \log(4 U(\lambda))}{n_{i,U(\lambda)}-1}}\right) -\frac{3}{n_{i,U(\lambda)}} - (1+\sqrt{2}) \sqrt{\frac{(c+2) \log(4 U(\lambda))}{n_{i,U(\lambda)}}} \\
& \geq  \lambda -\frac{3}{n_{i,U(\lambda)}}  - 2  (1+\sqrt{2}) \sqrt{\frac{(c+2) \log(4 U(\lambda))}{n_{i,U(\lambda)}-1}}.
\end{align*}
The first inequality comes from the fact that we are on event $\xi$ and we assume $U(\lambda) > S'$. The second inequality uses the fact that when we make a request to an expert, the number of items uniquely seen on this expert can drop by at most one, and thus we get
$$s \Rh_{i,s} \geq (s-1) \Rh_{i,s-1} - 1 \geq s \Rh_{i,s-1} - 2 .$$
The third inequality is the key step of the proof. Consider the time step $t$ such that $n_{i,t} = n_{i, U(\lambda)}-1$ and $n_{i,t+1} = n_{i, U(\lambda)}$. Since $t < U(\lambda)$ we know that one of the expert satisfies $\Rh_{j,n_{j,t}} + (1+\sqrt{2}) \sqrt{\frac{(c+2) \log(4 t)}{n_{j,t}}} > \lambda$. Moreover, since Good-UCB is run with constant $C=(1+\sqrt{2})\sqrt{c+2}$ and since we make a request to expert $i$ at time $t$, we know that it maximizes the Good-UCB index, and thus $\Rh_{i,n_{i,t}} + (1+\sqrt{2}) \sqrt{\frac{(c+2) \log(4 t)}{n_{i,t}}} > \lambda$. Using that $t \leq U(\lambda)$ completes the proof of the third inequality. The fourth inequality is trivial. 

We just proved that if $n_{i,U(\lambda)} > S'$ then
$$R_{i, n_{i,U(\lambda)}} \geq \lambda - \frac{3}{S'} - 2  (1+\sqrt{2}) \sqrt{\frac{(c+2) \log(4 U(\lambda))}{S'}} ,$$
which clearly implies
$$n_{i,U(\lambda)} \leq T_i^*\left(\lambda - \frac{3}{S'} - 2  (1+\sqrt{2}) \sqrt{\frac{(c+2) \log(4 U(\lambda))}{S'}} \right) .$$
Thus in any case we have proved that
$$n_{i,U(\lambda)} \leq S' + T_i^*\left(\lambda - \frac{3}{S'} - 2  (1+\sqrt{2}) \sqrt{\frac{(c+2) \log(4 U(\lambda))}{S'}} \right) ,$$
which implies
\begin{eqnarray*}
U(\lambda) & \leq & K S' + T^*\left(\lambda - \frac{3}{S'} - 2  (1+\sqrt{2}) \sqrt{\frac{(c+2) \log(4 U(\lambda))}{S'}} \right) \\
& \leq & K S \log(4 U(\lambda)) + T^*\left(\lambda - \frac{3}{S} - 2  (1+\sqrt{2}) \sqrt{\frac{c+2}{S}} \right) ,
\end{eqnarray*}
where the last inequality follows by taking $S' = S \log(4 U(\lambda))$. Finally using Lemma \ref{lem:logineq} (in the appendix) and ${T}_{UCB}(\lambda) \leq \max(S', U(\lambda))$ ends the proof.
\end{proof}

\section{Macroscopic Limit} \label{sec:macro}
In the previous section we derived a very general non-linear regret bound for Good-UCB. Here we shall study the behavior of Good-UCB under more restrictive assumptions on the experts, but it will allow us to derive a clear qualitative statement about its performance, and it also permits easier comparison with other strategies such as uniform sampling. In this section we shall add the two following assumptions in addition to assumption (i):
\begin{enumerate}
\item[(ii)] finite supports with the same cardinality: $|\supp(P_i)| = N, \forall i \in \{1,\hdots,K\}$,
\item[(iii)] uniform distributions: $P_i(x) = \frac{1}{N}, \forall x \in \supp(P_i), \forall i \in \{1,\hdots,K\}$.
\end{enumerate}
These assumptions are primarily made in order to be able to assess the performance of the optimal strategy. 
In this setting it is convenient to re-parameterize slightly the problem (in particular we make explicit the dependency on $N$ for reasons that will appear later). Let $\X^N = \{1,\dots,K\} \times \{1,\hdots,N\}$, $A^N\subset \X^N$ the set of interesting items of $\X^N$, and $Q^N = |A^N|$ the number of interesting items. We assume that, for expert $i \in \{1,\dots,K\}$,  $P^N_i$ is the uniform distribution on $\{i\}\times \{1,\hdots,N\}$. We also denote by $Q_i^N = \left|A^N \cap \left(\{i\}\times \{1,\hdots,N\}\right) \right|$ the number of interesting items accessible through requests to expert $i$. Without loss of generality, we assume in this section that $Q^N_1\geq Q^N_2\geq\dots\geq Q^N_K$.

The macroscopic limit that we investigate in this section corresponds to the setting where $N$ goes to infinity together with the $Q^N_i$ in such a way that $Q^N_i/N \to q_i\in (0,1)$. For a given strategy we are interested in the time $T^N(\lambda)$ such that all experts have at most $N\lambda$ undiscovered interesting items. In particular we define $T^N_{UCB}(\lambda)$ (respectively $T^N_{*}(\lambda)$) to be the corresponding time for the Good-UCB strategy (respectively the oracle omniscient strategy). In the macroscopic limit we shall be particularly interested in normalized limit waiting time $\lim_{N \to +\infty} T^N(\lambda) / N$. 

\subsection{Macroscopic Behavior of the Oracle Closed-loop Strategy}
In this section we shall derive an explicit upper bound on the macroscopic limit of $T_*^N$ by studying 
the OCL strategy introduced in Section \ref{sec:oclnew}. Recall
that at each time step, OCL makes a request to one of the experts with highest number of still undiscovered interesting items: the expert requested at time $t$ is:
\[I_t \in \argmax_{1\leq i\leq K} P_i\left(A \setminus \{X_{1,1},\hdots,X_{1,n_{1,t}}, \hdots, X_{K,1},\hdots,X_{K,n_{K,t}}\}\right)\;.\]

\begin{theorem} \label{th:OCL}
For every $\lambda\in (0, q_1)$, for every sequence $(\lambda^N)_N$ converging to $\lambda$ as $N$ goes to infinity, under assumption (i), (ii) and (iii), almost surely 
\[\lim_{N\to\infty}\frac{T_{OCL}^N(\lambda^N)}{N} = \sum_{i: q_i>\lambda} \log\frac{q_i}{\lambda}  \;.\] 
\end{theorem}

\begin{proof}
Denote by $B^N_i$ the set of interesting items in $\{1,\hdots, N\}$ supported by $P^N_i$: $B^N_i = \{x\in\{1,\hdots,N\} : (i,x)\in A^N\}$. Successive draws of expert $i$ are denoted $(i,X^N_{i,1}),(i,X^N_{i,2})\dots$ where the variables $(X^N_{i,n})_{i,n}$ are assumed to be independent.
Without loss of generality, we may assume that $N\lambda^N$ is a positive integer, for otherwise $\lambda^N$ can be replaced by $\lceil N\lambda^N\rceil/N$.
We denote by $(D^N_{i,k})_{1\leq k\leq Q^N_i}$ the increasing sequence of the indices corresponding to draws for which new interesting items are discovered with expert $i$: 
\[D^N_{i,1} = \min\left\{n\geq 1 : X^N_{i,n} \in B^N_i\right\},
\quad D^N_{i,2} = \min\left\{n\geq D^N_{i,1} : X^N_{i,n} \in B^N_i\setminus \left\{X^N_{i,D^N_{i,1}}\right\} \right\}, \dots\]
We also define $S^N_{i,0} = 0$ and for $k\geq 1, S^N_{i,k} = D^N_{i,k}-D^N_{i,k-1}$. The random variables $S^N_{i,k}$ ($1\leq i\leq K, k\geq 1$) are independent with geometric distribution $\G((1+Q^N_i-k)/N)$.

At every step, the OCL should call the expert with maximal number of undiscovered interesting items. Hence, it can:
\begin{itemize}
 \item first request expert $1$ for $D^N_{1,Q^N_1-Q^N_2}$ steps;
 \item then, alternatively request
 \begin{itemize}
   \item expert $1$ for $S^N_{1,1+Q^N_1-Q^N_2}$ steps;
   \item expert $2$ for $S^N_{2,{}1}$ steps;
   \item expert $1$ for $S^N_{1,2+Q^N_1-Q^N_2}$ steps;
   \item expert $2$ for $S^N_{2,2}$ steps;
   \item and so on, until there are only $Q^N_3$ undiscovered interesting items on experts $1$ and $2$.
 \end{itemize}
 \item and so on, including successively experts $3,4,\dots,K$ in the alternation.
\end{itemize}

Obviously, 
\[T_{OCL}^N(\lambda^N) = \sum_{i: Q^N_i>N\lambda^N} D^N_{i, Q^N_i-N\lambda^N}\;.\]
It suffices now to show that for every expert $i\in\{1,\dots,K\}$,  $D^N_{i, Q^N_i-N\lambda^N}/N$ converges almost surely to $\log(q_i/\lambda)$ as $N$ goes to infinity.
Write
\begin{equation}\label{eq:WsumS}
W^N_{i,N\lambda^N} = \frac{1}{N}\left(D^N_{i, Q^N_i-N\lambda^N} - \E\left[D^N_{i, Q^N_i-N\lambda^N}\right]\right) = 
\frac{1}{N}\sum_{k=1}^{Q^N_i-N\lambda^N-1} \left(S^N_{i, k} - \E\left[S^N_{i, k}\right]\right) \;.\end{equation}
For every positive integer $d$ and for $k\in\{1,\dots, N\lambda^N-1\}$, elementary manipulations of the geometric distribution yield that \[\E\left[\left(S^N_{i, k} - \E\left[S^N_{i, k} \right]\right)^d\right] \leq \E\left[\left(S^N_{i, N\lambda^N} - \E\left[S^N_{i, N\lambda^N} \right]\right)^d\right] \leq \frac{c(d)}{(\lambda^N)^d} \leq \frac{2c(d)}{\lambda^4}\]
for some positive constant $c(d)$ depending only on $d$, and for $N$ large enough. 
Hence, taking~\eqref{eq:WsumS} to the fourth power and developing yields
\[\E\left[\left(W^N_{i,N\lambda^N}\right)^4\right] \leq \frac{c'}{N^2\lambda^4}\] for some positive constant $c'$. Using Markov's inequality together with the Borel-Cantelli lemma, this permits to show that $W^N_{i,\lambda^N}$ converges almost surely to $0$ as $N$ goes to infinity.
But
\[ \frac{1}{N}\E\left[D^N_{i, Q^N_i-N\lambda^N}\right] = \frac{1}{Q^N_1} + \dots +\frac{1}{N\lambda^N+1} = \log\frac{Q^N_i}{N\lambda^N} - \epsilon^N\;, \]
with $0\leq \epsilon^N \leq 1/(N\lambda^N)$ according to Lemma~\ref{lem:harmo}, and thus 
\[\frac{1}{N}\E\left[D^N_{i, Q^N_i-N\lambda^N}\right]\to \lim_{N\to\infty} \log\left( \frac{Q^N_i/N}{\lambda^N} \right) = \log(q_i/\lambda)\;,\]
which concludes the proof.
\end{proof}

\subsection{Macroscopic Behavior of Uniform Sampling}
In this section we study the simple uniform sampling strategy that cycles through the experts, that is, at time $t$ uniform sampling makes a request to the $(t \mod [K])^{th}$ expert. This strategy is not macroscopically optimal unless all experts have the same number of interesting items. Furthermore the next proposition makes precise the extent of improvement of a macroscopic optimal strategy over uniform sampling. The proof follows the exact same steps than the proof of Theorem \ref{th:OCL} and thus is omitted.

\begin{proposition} \label{prop:unif}
For every $\lambda\in (0, q_1)$, for every sequence $(\lambda^N)_N$ converging to $\lambda$ as $N$ goes to infinity, under assumption (i), (ii) and (iii), almost surely
\[\lim_{N\to\infty}\frac{T_{US}^N(\lambda^N)}{N} = K \log\frac{q_1}{\lambda}\;.\]
\end{proposition}

\subsection{Macroscopic Optimality of Good-UCB}
Using the regret bound of Theorem \ref{th:nlregret} we obtain the following corollary that shows the asymptotic optimality of the Good-UCB algorithm in the macroscopic sense.

\begin{corollary}
Take $C=(1+\sqrt{2})\sqrt{c+2}$ with $c>3/2$ in Algorithm~\ref{algo:GUCB}.
Under assumption (i), (ii) and (iii), for every sequence $(\lambda^N)_N$ converging to $\lambda$ as $N$ goes to infinity, almost surely 
\[\limsup_{N \to +\infty} \frac{T_{UCB}^N(\lambda^N)}{N} \leq \sum_{i: q_i>\lambda} \log\frac{q_i}{\lambda}\;.\]
\end{corollary}

\begin{proof}
Let $S^N = N^{2/3}$. 
First note that:
\[\ell^N \stackrel{def}{=}\lambda^N - \frac{3}{S^N} - 2  (1+\sqrt{2}) \sqrt{\frac{c+2}{S^N}} \to \lambda \quad \hbox{ when } N \to \infty\;.\]
Thus, by Theorem~\ref{th:OCL}, and the fact that the OCL strategy needs at least as much time as the omniscient oracle strategy in order to find the same number of items, there exists an event $\Omega$ of probability $1$ on which
\[\limsup_{N \to +\infty} \frac{T_*^N \left(\ell^N\right)}{N} \leq \sum_{i: q_i>\lambda} \log\frac{q_i}{\lambda}\;.\]
Thus, according to Theorem~\ref{th:nlregret}, for each positive integer $N$ there exists an event $A_N$ of probability $P(A_N)\geq 1-K/(cN^{2c/3})$ on which
\begin{align*}
 \frac{T^N_{UCB}(\lambda^N)}{N} &\leq \frac{T_*^N\left(\ell^N \right)}{N} + \frac{K S^N}{N}\log \left(8 T_*^N\left(\ell^N\right) + 16 K S \log( K S^N ) \right)\\
& = \frac{T_N^*\left(\ell^N \right)}{N}  + O\left(\frac{\log(N)}{N^{1/3}}\right)\;.
\end{align*}
Using Borel-Cantelli's lemma and the fact that, with our choice of parameters, $\sum_{N} N^{-2c/3} <\infty$, we obtain that except maybe on the set (of probability $0$) $\bar{\Omega} \cup \limsup \overline{A_N}$, 
\[ \limsup_{N\to\infty} \frac{T^N_{UCB}(\lambda^N)}{N} \leq  \limsup_{N \to +\infty} \frac{T^N_*(\ell^N)}{N} \leq \sum_{i: q_i>\lambda} \log\frac{q_i}{\lambda}\;,\]
which ends the proof.
\end{proof}

\section{Simulations}\label{sec:simus}
We provide a few simulations illustrating the behavior of the Good-UCB algorithm and the asymptotic analysis above of Section~\ref{sec:macro}. We first consider an example with $K=7$ different sampling distributions satisfying assumptions [(i),(ii),(iii)], with respective proportions of interesting items $q_1 = 51.2\%, q_2 = 25.6\%, q_3 = 12.8\%, q_4 = 6.4\%, q_5 = 3.2\%, q_6 = 1.6\%$ and $ q_7 = 0.8\%$.

We have chosen to display here the numbers of items found as a function of the number of draws (see \eqref{eq:numberfound}), instead of the times $T^N(\lambda^N)$, because they express more intuitively the discovering possibilities of each algorithm. Note, however, that the correspondence between these two quantities is straightforward, especially in the macroscopic limit: For $\lambda\in(0, q_1)$ let
\begin{equation}\label{eq:defT}
 T(\lambda)=\sum_{i:q_i>\lambda} \log\frac{q_i}{\lambda}\;.
\end{equation}
It is easy to show that the proportion of interesting items found by the OCL strategy after $N t$ draws converge to 
\begin{equation}\label{eq:defF}
F(t) = \sum_{i=1}^K \left(q_i-T^{-1}(t)\right)_+\;.\end{equation}
Furthermore the latter expression is a lower bound for the corresponding proportion of interesting items found by the Good-UCB algorithm. Proposition~\ref{prop:roo}, proved in the Appendix, provides a more explicit expression for $F$: denoting $q = \sum_{i=1}^K q_i$, there exists an increasing, $\{1,\dots,K\}$-valued function $I$ such that, for each $t$, 
\[ F(t) =  q -  I(t) \uq_{I(t)} \exp\left( -t/I(t) \right) \;,\]
where $\uq_{I(t)}$ denotes the geometric mean of $q_1,\dots,q_{I(t)}$.
This permits an explicit comparison of the macroscopic performance of the Good-UCB algorithm with uniform sampling: when all distributions are sampled equally often, the proportion of unseen interesting items at time $t$ is smaller than
\[\sum_{i=1}^K q_i\exp(-t/K) = K\bar{q}_K\exp(-t/K)\;,\]where $\bar{q}_K = (\sum_{i=1}^K q_i)/K$ is the arithmetic mean of the $(q_i)_i$.
On the other hand, for the Good-UCB algorithm, the proportion of unseen interesting items at time $t$ is smaller than
\[I(t) \uq_{I(t)} \exp\left( -t/I(t) \right)\;.\]
The ratio of those two quantities is a decreasing function of time lower-bounded by $\bar{q}_K / \uq_K\geq 1$, the ratio of the arithmetic mean with the geometric mean of the  $(q_i)_i$. As expected, this ratio gets larger when the proportions of interesting items among experts becomes more unbalanced.

Figure~\ref{fig:nbFound} displays the number of items found as a function of time by the Good-UCB (solid), the OCL (dashed) and the uniform sampling scheme that alternates between experts (dotted). The results are presented for sizes $N=128, N=500, N=1000$ and $N=10000$, each time for one representative run (averaging over different runs removes the interesting variability of the process). We chose to plot the number of items found rather than the waiting time $t$ as the former is easier to visualize while the latter was easier to analyze. In fact, macroscopic optimality in terms of number of items found could also be derived with the techniques of Section \ref{sec:macro}. Figure~\ref{fig:nbFound} also shows clearly the macroscopic convergence of Good-UCB to the OCL.
Moreover, it can be seen that, even for very moderate values of $N$, the Good-UCB significantly outperforms uniform sampling even if it is clearly distanced by the OCL.

For these simulations, the parameter $C$ of Algorithm Good-UCB  has been taken equal to $1/2$, which is a rather conservative choice. In fact, it appears that during all rounds of all runs, all upper-confidence bounds did contain the actual missing mass. Of course, a bolder choice of $C$ can only improve the performance of the algorithm, as long as the confidence level remains sufficient.

\begin{figure}

  \includegraphics[width=7.8cm]{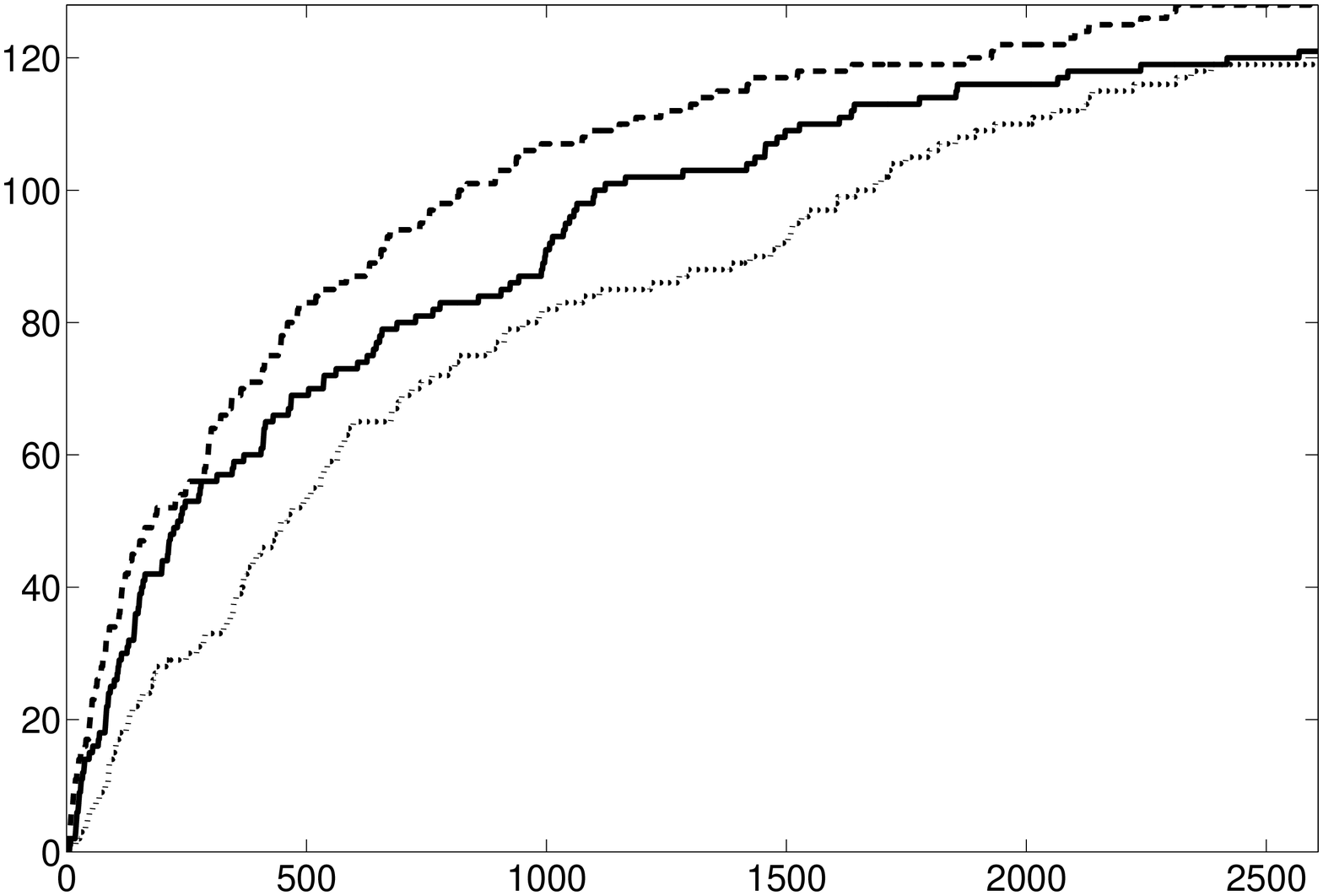}
  \includegraphics[width=7.8cm]{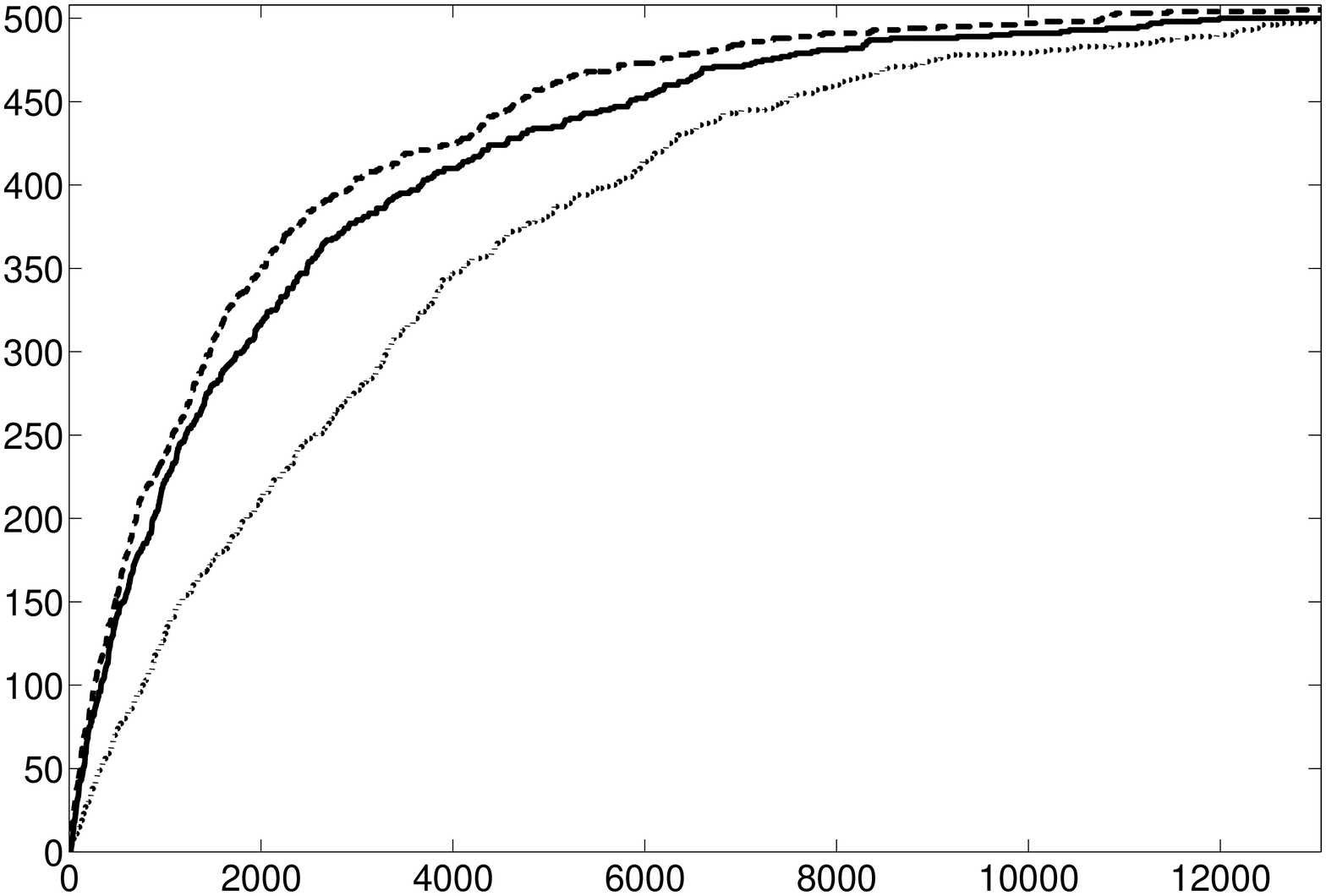} 
 
  \includegraphics[width=7.8cm]{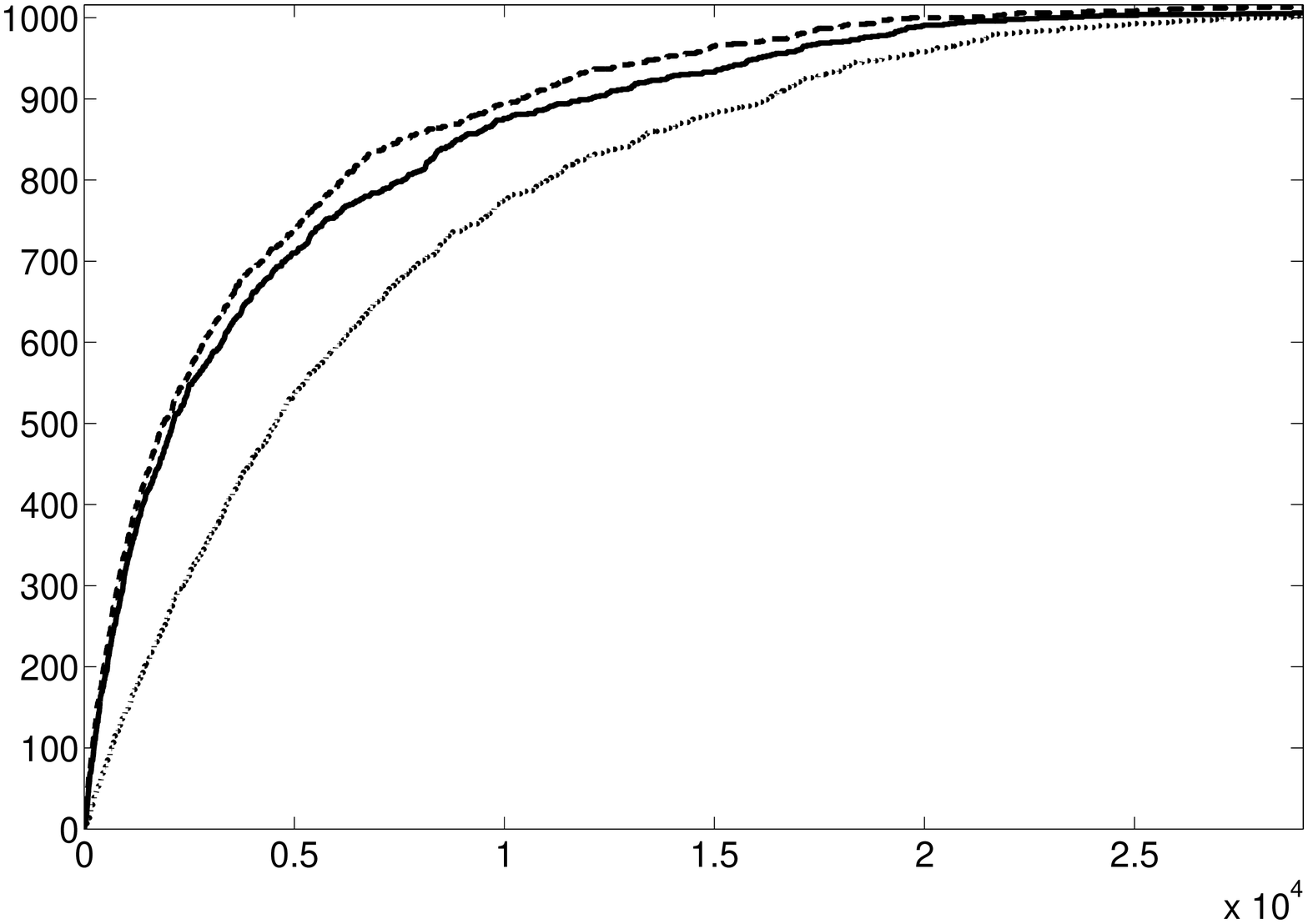}
  \includegraphics[width=7.8cm]{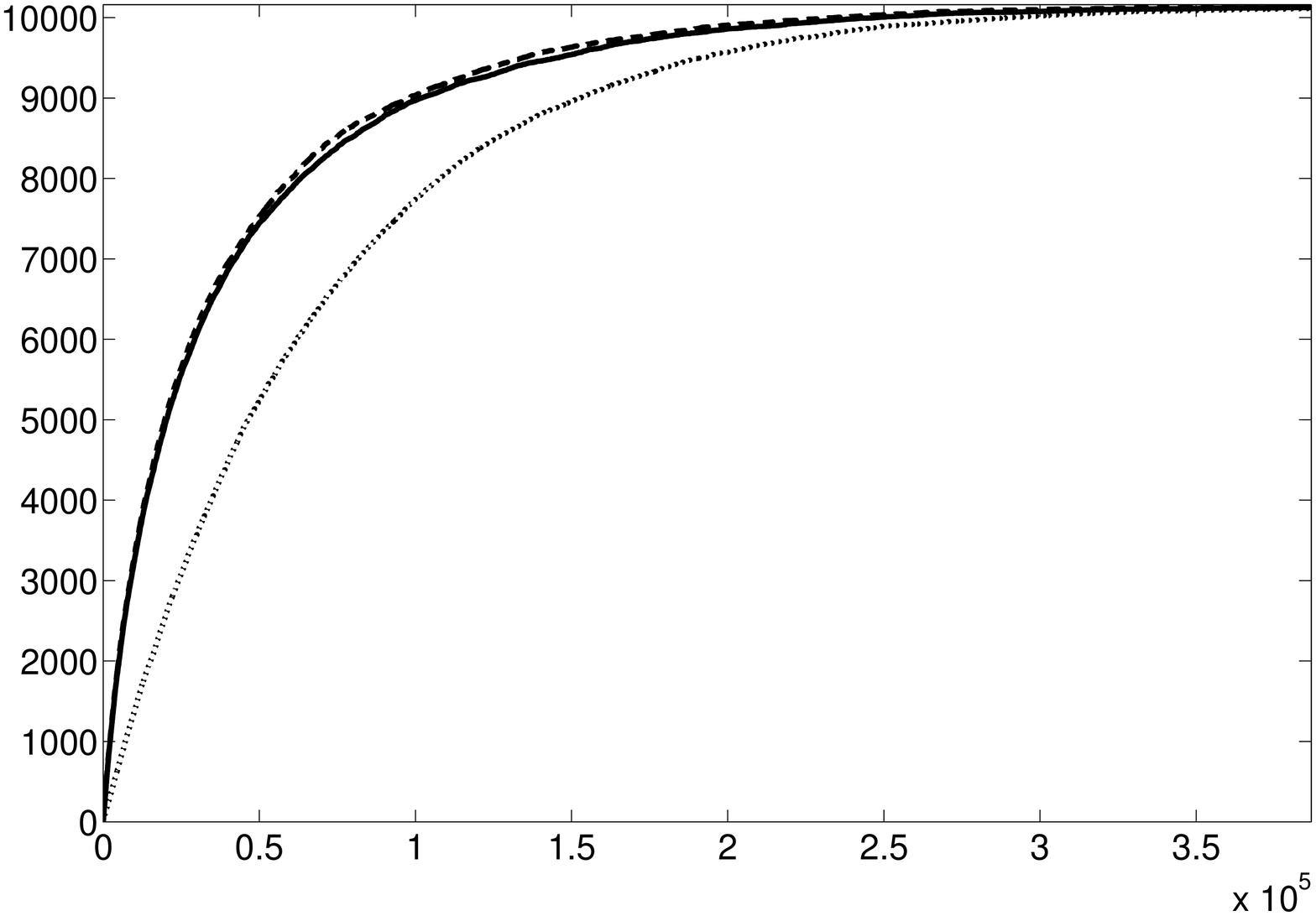}

  \caption{Number of items found by Good-UCB (solid), the OCL (dashed), and uniform sampling (dotted) as a function of time for sizes $N=128, N=500, N=1000$ and $N=10000$ in a $7$-experts setting.}
  \label{fig:nbFound}
  \end{figure}
\begin{figure}
  \includegraphics[width=7.8cm]{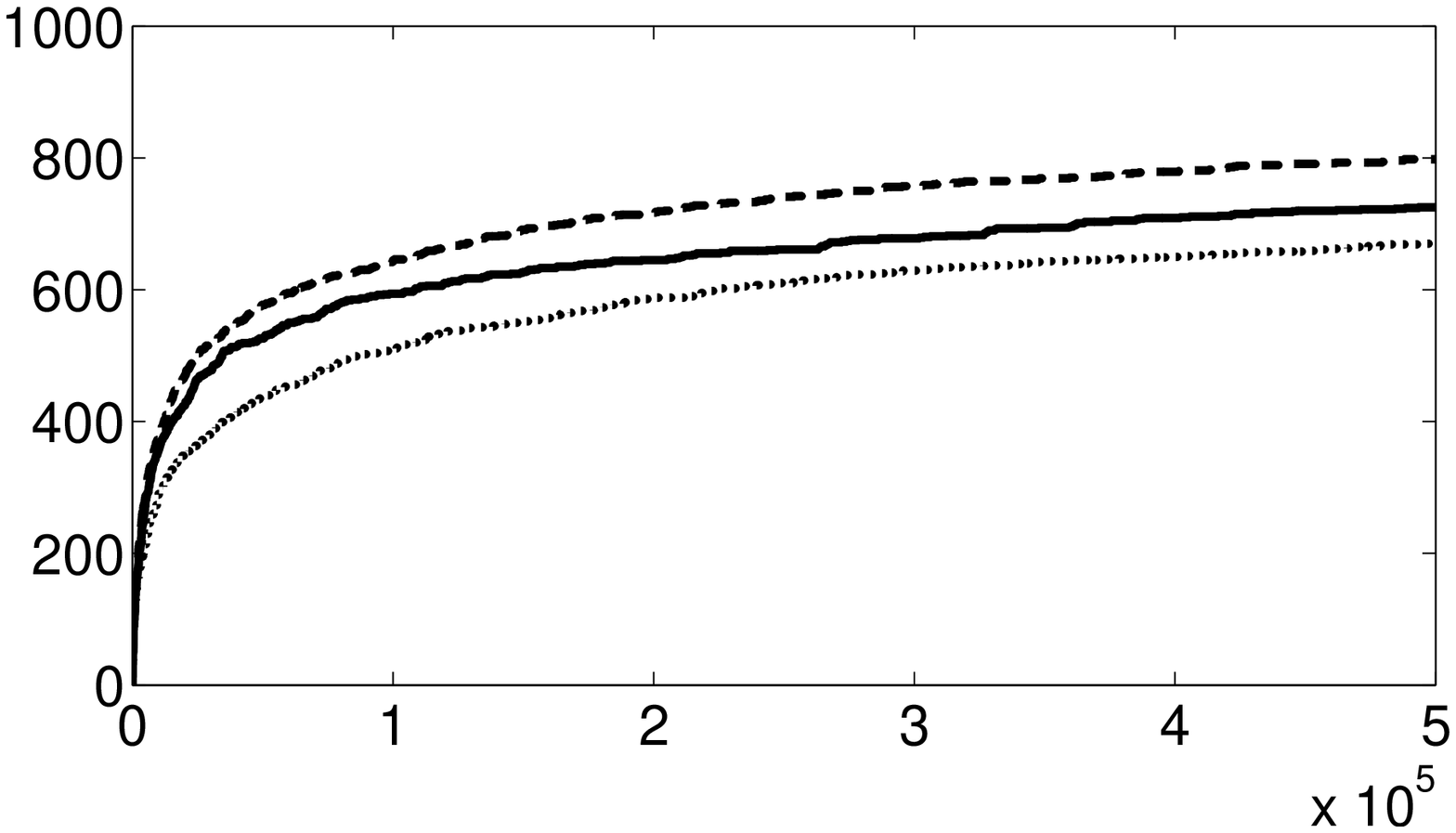}
  \includegraphics[width=7.8cm]{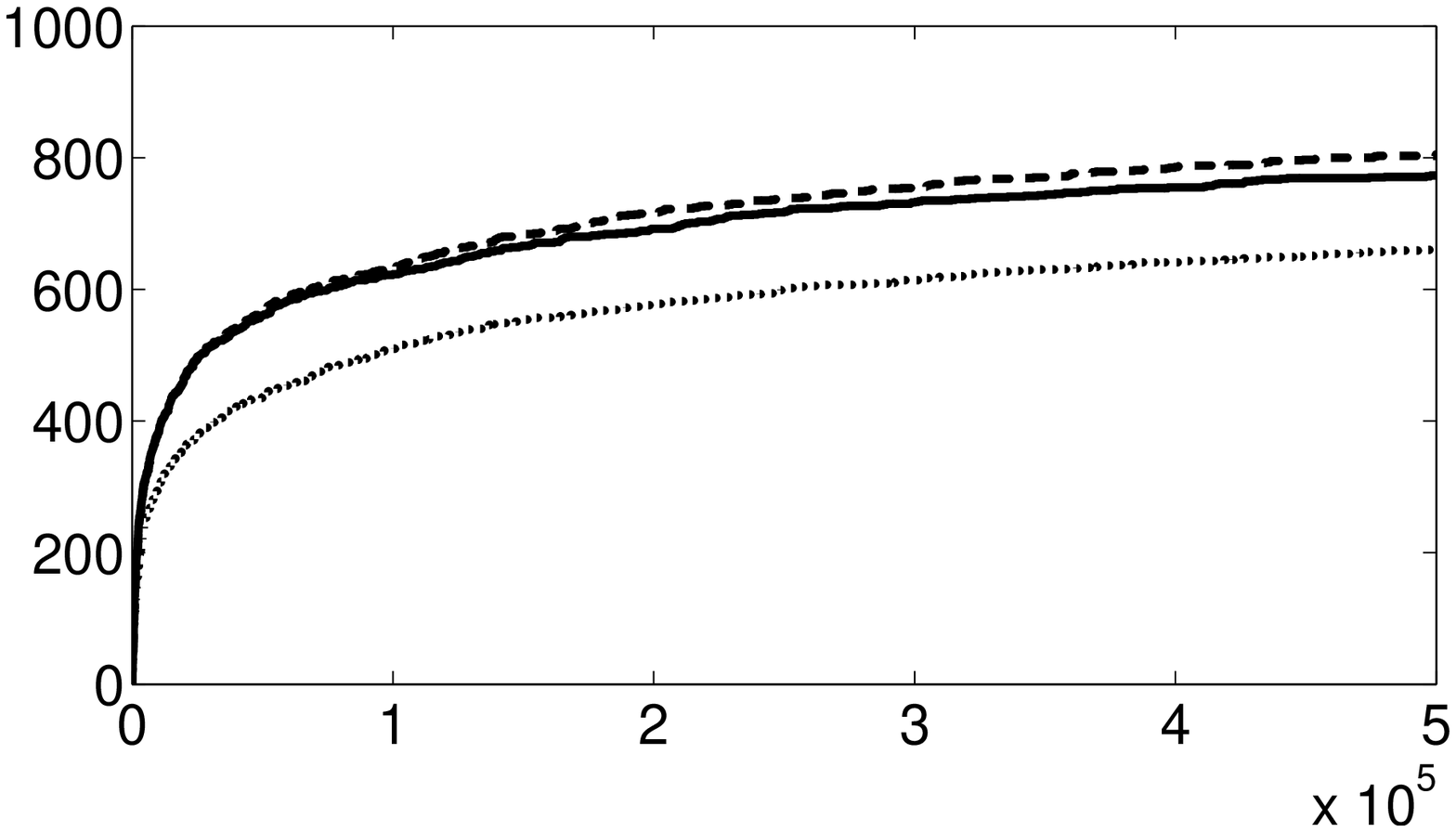}
 \caption{Number of prime numbers found by Good-UCB (solid), the OCL (dashed), and uniform sampling (dotted) as a function of time, using geometric experts with means $100, 300, 500, 700$ and $900$, for $C=0.1$ (left) and $C=0.02$ (right).}
  \label{fig:primeNumbers}
  \end{figure}

In order to illustrate the efficiency of the Good-UCB algorithm in a more difficult setting, which does not satisfy any of the assumptions (i), (ii) and (iii), we also considered the following (artificial) example: $K=5$ probabilistic experts draw independent sequences of geometrically distributed random variables, with expectations $100$, $300$, $500$, $700$ and $900$ respectively. The set of interesting items is the set of prime numbers.
We compare the oracle closed-loop policy, Good-UCB and uniform sampling.
The results are displayed in Figure~\ref{fig:primeNumbers}. Even if the difference remains significant between Good-UCB and the OCL, the former still performs significantly better than uniform sampling during the entire discovery process. 
In this example, choosing a smaller parameter $C$ seems to be preferable; this is due to the fact that the proportion of interesting items on each arm is low; in that case, it may be possible to show, by using tighter concentration inequalities, that the concentration of the Good-Turing estimator is actually better than suggested by Proposition~\ref{prop:inegGT}. In fact, this experiment suggests that the value of $C$ should be chosen smaller when the remaining missing mass is small. 

\section*{Acknowledgements}
We are especially thankful to one of the anonymous referees for suggesting to us to write Sections~\ref{sec:oclnew} and~\ref{sec:classical}.

\appendix

\section{Proofs of Technical Lemmas}

\begin{proof}\textbf{of lemma~\ref{lem:OCLoptimal}}
We proceed by induction on $t$. For $t=1$, the result is obvious. 
For $t>1$, denote by $\pb$ the policy choosing $I^{\pb}_1 = I_1^{\pi}$ and then playing like OCL for the $t-1$ remaining rounds. 
Denote $H_1 = (I_1^{\pi}, X_{I_1^{\pi}, 1})$, and $F^\pi(2:t)$ (respectively $F^{\pb}(2:t)$) the number of interesting items found by policy $\pi$ (respectively $\pb$) between rounds $2$ and $t$. Note that conditionally on $H_1$, $F^{\pb}(2:t)$ corresponds to $F^*(t-1)$ in some modified problem (where one interesting item on expert $I_1^{\pi}$ might have been removed from the set of interesting items). Thus one can apply the induction hypothesis to obtain
$$\E\left[ F^{\pi}(2:t) | H_1\right] \leq \E\left[ F^{\pb}(2:t) | H_1\right] .$$
Let us assume in the following that $I_1^{\pi}$ is deterministic (we make this assumption only for sake of clarity, everything go through with a randomized choice of $I_1^{\pi}$). Then thanks to the above inequality one has
\begin{equation}\label{eq:lm2:0}
 \E F^\pi(t) = R_{I_1^{\pi},0} + \E\left[ F^{\pi}(2:t)\right] \leq R_{I_1^{\pi},0} + \E\left[ F^{\pb}(2:t)\right] = \E F^{\pb}(t)\;.
\end{equation}
Now let \[
     \tau = \min \{s \geq 1 : I_{s}^* = I_1^{\pi}\} \;.
    \]
On the event $\tau\leq t$, OCL and $\pb$ observe exactly the same items during the $t$ first rounds, and thus 
\begin{equation}\label{eq:lm2:1}F^{\pb}(t)\1\{\tau\leq t\}= F^*(t) \1\{\tau\leq t\}\;.\end{equation}
On the other hand on the event $\tau>t$, $\pb$ observe the same items between rounds $2$ and $t$ than OCL between rounds $1$ and $t-1$, that is $F^{\pb}(2:t) \1\{\tau>t\} = F^*(t-1) \1\{\tau>t\}$. Thanks to assumption (i), this implies (denoting $Y^*_1, \hdots, Y^*_{t}$ for the sequence of items observed by OCL),
\begin{equation}
 F^{\pb}(t)\1\{\tau>t\} =  \left(\1\{X_{I_1^{\pi},1} \in A\} + F^*(t) - \1\{Y^*_t \in A \setminus \{Y^*_1, \hdots, Y^*_{t-1}\}\} \right)\1\{\tau>t\} \;. \label{eq:lm2:2}
\end{equation}
By combining \eqref{eq:lm2:0}, \eqref{eq:lm2:1} and \eqref{eq:lm2:2}, it only remains to show that
\begin{equation}
\E [ \1\{X_{I_1^{\pi},1} \in A\} \1\{\tau>t\} ] \leq \E [ \1\{Y^*_t \in A \setminus \{Y^*_1, \hdots, Y^*_{t-1}\}\} \1\{\tau>t\} ] . \label{eq:lm2:3}
\end{equation}
Since $X_{I_1^{\pi}, 1}$ is independent of $\1\{\tau>t\}$, one has $\E [ \1\{X_{I_1^{\pi},1} \in A\} \1\{\tau>t\} ] = \E [ R_{I_1^{\pi},0} \1\{\tau>t \}]$. Moreover,
noting that $\1\{\tau>t\}$, $I_t^*$ and $R_{I_t^*, n^*_{I_t^*, t-1}}$ are measurable with respect to $H^*_{t-1} = \left(I^*_1, Y^*_1,\dots, I^*_{t-1}, Y^*_{t-1}\right)$, one has
$$\E [ \1\{Y^*_t \in A \setminus \{Y^*_1, \hdots, Y^*_{t-1}\}\} \1\{\tau>t\} ] = \E [R_{I_t^*, n^*_{I_t^*, t-1}} \1\{\tau>t\} ]  .$$
Finally remark that on the event $\tau>t$ one necessarily have that the remaining missing mass on the expert pulled at time $t$ by OCL is larger than the initial missing mass of expert $I_1^{\pi}$, that is $R_{I_t^*, n^*_{I_t^*, t-1}} \1\{\tau>t\} \geq R_{I_1^{\pi},0} \1\{\tau>t\}$, which concludes the proof of \eqref{eq:lm2:3}.
\end{proof}

\begin{proof}\textbf{of Lemma~\ref{lem:firstBound}}
Let  $Y_s^{\pi} = X_{I_s^{\pi}, n_{I^{\pi}_s, s}}$ be the item observed by $\pi$ at time step $s$, and \linebreak[4] $H_s^{\pi} = \big(I_1^{\pi}, Y_1^{\pi}, \dots, I^{\pi}_{s-1}, Y^{\pi}_{s-1})$ be the history of $\pi$ prior to making the decision on time $s$. For any history $h_s = (i_1,y_1,\dots,i_{s-1},y_{s-1})$, let $F^*(t|h_s)$ be the number of newly discovered interesting items when running OCL from the history $h_s$ for $t-s+1$ steps. 'From the history $h_s$' means that, prior to running OCL, the sequence of experts $i_1,\dots,i_{s-1}$ has been chosen and has led to the observations $y_1,\dots,y_{s-1}$. For $s' \geq s$ we shall also denote $I_{s'}^*(h_s)$ (respectively $Y_{s'}^*(h_s)$) the sequence of expert requests made by OCL starting at $h_s$ (respectively the corresponding sequence of observed items). Note in particular that $\bar{I}_s$ defined in the statement of the lemma corresponds to $I_s^*(H_s^{\pi})$. We shall also need $\tau_s$ to be the first time when OCL, running from history $H_s^{\pi}$, selects 
expert $I_s^{\pi}$, that is
$$\tau_s = \min \{s' \geq s : I_{s'}^*(H_s^{\pi}) = I_s^{\pi}\} \;,$$
and $\tau_s = +\infty$ if there is no interesting item to be found by expert $I_s^\pi$ at time $s$.

We shall prove that 
\begin{equation} \label{eq:trucaprouver}
\E [ F^*(t|H_{s}^{\pi}) - F^*(t|H_{s+1}^{\pi}) ] \leq \E R_{\bar{I}_s, n^{\pi}_{\bar{I}_s,s-1}} ,
\end{equation}
which inductively yields the lemma since $F^*(t) = F^*(t|h_1)$ and $F^*(t|h_{t+1}) = 0$.

First let us consider the case when $\tau_s \leq t$. Then the observed items with OCL (running from $H_s^{\pi}$) between step $s$ and $t$ remains unchanged if one forces OCL to play $I_s^{\pi}$ at time step $s$, that is
$$F^*(t | H_s^{\pi}) \1\{\tau_s \leq t\} =  \1\{{Y}_s^{\pi} \in A \setminus \{Y^{\pi}_1, \hdots, Y^{\pi}_{s-1}\}\} \1\{\tau_s \leq t\}  + F^*(t|H_{s+1}^{\pi}) \1\{\tau_s \leq t\} .$$
On the other hand if $\tau_s > t$, the behavior of OCL will be the same if played for $t-s$ steps from $H_{s}^{\pi}$ or from $H_{s+1}^{\pi}$, that is 
$$F^*(t-1|H_{s}^{\pi}) \1\{\tau_s > t\} = F^*(t|H_{s+1}^{\pi}) \1\{\tau_s > t\} .$$
Moreover note that
$$F^*(t-1|H_{s}^{\pi}) = F^*(t|H_{s}^{\pi}) - \1\{{Y}_t^{*}(H_s^{\pi}) \in A \setminus \{Y^{\pi}_1, \hdots, Y^{\pi}_{s-1}, Y^*_s(H_s^{\pi}), \hdots, Y^*_{t-1}(H_s^{\pi})\}\} .$$
Thus we proved so far that
\begin{align*}
& F^*(t|H_{s}^{\pi}) - F^*(t|H_{s+1}^{\pi}) \\
& = \1\{{Y}_s^{\pi} \in A \setminus \{Y^{\pi}_1, \hdots, Y^{\pi}_{s-1}\}\} \1\{\tau_s \leq t\} \\
& \qquad + \1\{{Y}_t^{*}(H_s^{\pi}) \in A \setminus\{Y^{\pi}_1, \hdots, Y^{\pi}_{s-1}, Y^*_s(H_s^{\pi}), \hdots, Y^*_{t-1}(H_s^{\pi})\}\} \1\{\tau_s > t\} \\
& \leq \1\{{Y}_s^{\pi} \in A \setminus \{Y^{\pi}_1, \hdots, Y^{\pi}_{s-1}\}\} \1\{\tau_s \leq t\} + \1\{{Y}_t^{*}(H_s^{\pi}) \in A \setminus \{Y^{\pi}_1, \hdots, Y^{\pi}_{s-1}\}\} \1\{\tau_s > t\}.
\end{align*}
Now remark that ${Y}_s^{\pi}$ is independent of $\tau_s$ conditionally to $H_s^{\pi}$. Thus one immediately obtains
\begin{align*}
& \E [\1\{{Y}_s^{\pi} \in A \setminus \{Y^{\pi}_1, \hdots, Y^{\pi}_{s-1}\}\} \1\{\tau_s \leq t\} | H_s^{\pi}] \\
& = R_{{I}_s, n^{\pi}_{{I}_s,s-1}} \E [\1\{\tau_s \leq t\} | H_s^{\pi}] \\
& \leq R_{\bar{I}_s, n^{\pi}_{\bar{I}_s,s-1}} \E [\1\{\tau_s \leq t\} | H_s^{\pi}] .
\end{align*}
Similarly ${Y}_t^{*}(H_s^{\pi})$ is independent of $\1\{\tau_s > t\}$ conditionally to $(H_s^{\pi}, I_t^*(H_s^{\pi}))$ and thus
\begin{align*}
& \E [\1\{{Y}_t^{*}(H_s^{\pi}) \in A \setminus \{Y^{\pi}_1, \hdots, Y^{\pi}_{s-1}\}\} \1\{\tau_s > t\} | H_s^{\pi}, I_t^*(H_s^{\pi})] \\
& = \E[ R_{{I}_t^*(H_s^{\pi}), n^{\pi}_{{I}_t^*(H_s^{\pi}),s-1}} | H_s^{\pi}, I_t^*(H_s^{\pi})] \E [\1\{\tau_s > t\} | H_s^{\pi}, I_t^*(H_s^{\pi})] \\
& \leq R_{\bar{I}_s, n^{\pi}_{\bar{I}_s,s-1}} \E [\1\{\tau_s > t\} | H_s^{\pi}, I_t^*(H_s^{\pi})] .
\end{align*}
Putting everything together one obtains \eqref{eq:trucaprouver}, which concludes the proof.
\end{proof}

\begin{lemma} \label{lem:logineq}
Let $a>0$, $b\geq0.4$, and $x \geq e$, such that $x \leq a + b \log x$. Then one has 
$$x \leq a + b \log \left(2 a + 4 b \log(4 b) \right).$$
\end{lemma}

\begin{proof}
If $a \geq b \log x$ then $x \leq 2 a$ and thus $x \leq a + b \log (2 a)$. On the other hand if $a < b \log x$ then $x \leq 2b \log x$ which easily implies $x \leq 4 b \log(4 b)$ (indeed for $x \geq e$, $x \mapsto \frac{x}{\log x}$ is increasing and furthermore for $b\geq0.4$ one can check that $4 b \log(4 b) > 2 b \log (4 b \log(4 b))$) and thus $x \leq a + b \log(4 b \log(4b))$. In any case one has $x \leq a + b \log \left(2 a + 4 b \log(4 b) \right)$.
\end{proof}

\begin{lemma}\label{lem:harmo}
For all $1\leq k\leq n$, 
\begin{equation*}
-\frac{1}{k}+\log\frac{n}{k}  \leq \sum_{j=k+1}^n \frac{1}{j} \leq \log\frac{n}{k}\;.
\end{equation*}
\end{lemma}

\begin{proof}
The standard sum/integral comparison yields
\[\log\frac{n+1}{k+1}\leq\sum_{j=k+1}^n \frac{1}{j} \leq \log\frac{n}{k}\; , \]
but 
\[\log\frac{n+1}{k+1} = \log\frac{n}{k} + \log\left( 1+\frac{1}{n+1} \right) - \log\left( 1+\frac{1}{k+1} \right)  \geq \log\frac{n}{k} + 0 -\frac{1}{k}\;. \]
\end{proof}

\section{The Open-loop Oracle Policy}\label{sec:OOL}
In this final section, we provide an macroscopic analysis of the open-loop oracle policy in the case of uniform sampling, that is under Hypotheses (i), (ii) and (iii).
An open-loop policy must choose, for each horizon $t$, the respective numbers of requests $(n^N_1,\dots,n^N_K)$ for each distribution (so that $n^N_1+\dots+n^N_K = t^N$) in advance.
It appears here that, in the limit, the \emph{oracle open-loop} (OOL) policy, which makes use of the parameters $(Q^N_1,\dots,Q^N_K)$, is as good as the OCL policy.

Let here $\underline{R}^N_{i,n^N_i}= (Q^N_i - F^N_{i}(n_i^N))/N$ be the proportion of interesting items not yet found with expert $i$ after $n^N_i$ requests.
Suppose that $t^N/N\to t$, and that $n^N_i/N\to \nu_i$ as $N$ goes to infinity; it is easily shown that, almost surely,
\[ \lim_{N\to\infty} \underline{R}^N_{i,n_i^N} =\lim_{N\to\infty} \E\left[\underline{R}^N_{i,n_i^N} \right] = \lim_{N\to\infty} \frac{ Q^N_i\left(1-\frac{1}{N}\right)^{n^N_i}}{N} = q_i\exp(-\nu_i)\;.\]
Hence, the proportion of interesting items found with the allocation $(n^N_1,\dots,n^N_K)$ almost surely converges to $\sum_{i=1}^K q_i\left( 1-\exp(-\nu_i) \right)$.
Defining 
\[r(\nu)= \sum_{i=1}^K q_i\exp(-\nu_i)\;,\] it follows that finding the best macroscopic allocation reduces to the following constrained convex minimization problem:
\[\min_{\nu\in\R^K} r(\nu) \hbox{\quad such that } \nu_1+\dots+\nu_K = t \hbox{ and } \forall i,\,\nu_i\geq 0 \;.\]
The solution $r^*(t)$, reached at $\nu=\nu^*(t)$, is easily derived by classical optimization techniques:
\begin{proposition}\label{prop:roo}
For every $i\in\{1,\dots,K\}$, let 
$\uq_i = \exp\left(1/i\times\sum_{k=1}^i \log q_k \right)$ denotes the geometric mean of $q_1,\dots,q_{i}$.
\begin{enumerate}
\item There exists $I(t)\in\{1,\dots, K\}$ such that 
\[\begin{cases}
\forall i\leq I(t),& \nu_i^*(t) =   \frac{t}{I(t)} + \log\frac{q_i}{\uq_{I(t)}} \\
\forall i> I(t),& \nu_i^*(t) = 0 \;.
\end{cases}\]
Hence,
\[r^*(t) = I(t) \uq_{I(t)} \exp\left( -\frac{t}{I(t)} \right) + \sum_{i>I(t)} q_i\;.\]
\item There exists $1=t_1\leq \dots \leq t_K<+\infty$ such that \[\forall t\in[ t_{i},  t_{i+1}[, \;I(t)=i\;.\] 
The $(t_k)_k$ are such that
\[q_{i} + (i-1)\uq_{i-1}\exp\left(-\frac{t_{i}}{i-1}\right) = i\uq_{i} \exp\left(-\frac{t_i}{i}\right)\;.\]
For instance, $t_1 = \log(q_1/q_2)$.
\end{enumerate}
\end{proposition}
\textbf{Proof:}
Introduce the Lagrangian:
\[L(\nu_1,\dots,\nu_K,\lambda, \mu_1,\dots, \mu_K) = \sum_{i=1}^K q_i\exp\left( -\frac{\nu_i}{N} \right) + \lambda \left( \sum_{i=1}^K \nu_i \right) -\sum_{i=1}^K \mu_i \nu_i\;.\]
We need to find the solution of:
\begin{align*}
\forall i\in\{1,\dots, M\},&\;\; -q_i\exp\left( -\nu_i \right) +\lambda - \mu_i = 0,\\
&\;\;\sum_{i=1}^K \nu_i = t,\\
\forall i\in\{1,\dots,M\},& \;\;\mu_i \nu_i = 0 \hbox{ and } \mu_i \geq 0\;.
\end{align*}
We first obtain that 
\[\nu_i =\log q_i -\log(\lambda - \mu_i)\;.\]
Denoting $A=\{i : \nu_i>0\}$, and using that $i\in A\implies \mu_i=0$, we get
\[t = \sum_{i\in A} \log(q_i) - |A| \log(\lambda)\;,\]
from which we get
\[-\log(\lambda) = \frac{t}{|A|} - \frac{1}{|A|}\sum_{i\in A}\log q_i, \]
and then for all $i\in A$:
\[\nu_i = \log q_i +\frac{t}{|A|} - \frac{1}{|A|}\sum_{i\in A}\log q_i\;.\]
Next, observe that $\nu_i=0 \iff q_i > \lambda$: in fact, if $\nu_i=0$ then the first equation gives $-q_i +\lambda -\mu_i = 0$, and $0\leq \mu_i = \lambda-q_i$.
Conversely, if $\nu_i>0$ then $\mu_i=0$ and $\nu_i = \log(q_i/\lambda)>0$ implies $q_i> \lambda$.
Thus, there exists $I(t)$ such that $A=\{1,\dots, I(t)\}$, and for all $i\leq I(t)$, 
\[\nu_i = \log\frac{q_i}{\uq_{I(t)}} + \frac{t}{I(t)}\;.\]
Moreover, 
\begin{align*}
r^*(t)  &= r\left(\nu_1,\dots, \nu_{I(t)}, 0,\dots, 0\right) \\
 &= \sum_{i\leq I(t)} q_i \exp\left[-\left(\log\frac{q_i}{\uq_{I(t)}} + \frac{t}{I(t)}\right)\right] +\sum_{i>I(t)} q_i\\
 &= I(t) \uq_{I(t)} \exp\left( -\frac{t}{I(t)} \right)+\sum_{i>I(t)} q_i\;.
\end{align*}
The instants $(t_i)_{1\leq i\leq K}$ are such that 
\[ (i-1) \uq_{i-1} \exp\left( -\frac{t_i}{i-1} \right)+\sum_{k>i-1} q_k =  i \uq_{i} \exp\left( -\frac{t_i}{i} \right)+\sum_{k>i} q_k\;, \]
which is equivalent to
\[q_i + (i-1)\uq_{i-1}\exp\left(-\frac{t_i}{i-1}\right) = i\uq_i\exp\left(-\frac{t_i}{i}\right)\;.\]
For $i=2$, this gives
\[0 = q_2 + q_1\exp(-\nu_2) - 2\sqrt{q_1q_2} \exp\left( -\frac{\nu_2}{2} \right) = \left( \sqrt{q_2} -\sqrt{q_1\exp\left( -\nu_2 \right)} \right)^2\;,\]
which leads to  $t_1 = \log(q_1/q_2)$.

\begin{theorem}\label{th:OOLoptimal}
In the macroscopic limit, the proportion of items found by the open-loop oracle policy uniformly converges to the function $F$ defined in Equation~\eqref{eq:defF}. 
\end{theorem}
The proportion of interesting items found by the OOL policy is 
\[q - r^*(t) = \sum_{i\leq I(t)} \left[ q_i - \uq_{I(t)}\exp\left( -\frac{t}{I(t)} \right)  \right]
=\sum_{i=1}^K \left( q_i - \Lambda(t) \right)_+
\;,\]
where $\Lambda(t) = \uq_{I(t)}\exp\left( -\frac{t}{I(t)} \right)\in [0, q_{I(t)}]$.
To conclude, it remains only to remark that $\Lambda = T^{-1}$, where $T$ is defined in Equation~\eqref{eq:defT}. In fact, if $\lambda$ is such that $q_{i_0+1}< \lambda\leq q_{i_0}$, then
$I(T(\lambda)) = i_0$ and 
\begin{equation*}
\Lambda\left( T(\lambda) \right) = \uq_{i_0}\exp\left( -\frac{T(\lambda)}{i_0}\right) \\
 = \exp\left( \frac{1}{i_0}\sum_{i\leq i_0} \log q_i \right) \exp\left( -\frac{\sum_{i\leq i_0} \log(q_i / \lambda)}{i_0}\right) = \lambda\;.
\end{equation*}
If $\lambda<q_K$, the same holds with $i_0=K$.

\bibliographystyle{plainnat}
\bibliography{biblio}

\begin{thebibliography}{11}
\providecommand{\natexlab}[1]{#1}
\providecommand{\url}[1]{\texttt{#1}}
\expandafter\ifx\csname urlstyle\endcsname\relax
  \providecommand{\doi}[1]{doi: #1}\else
  \providecommand{\doi}{doi: \begingroup \urlstyle{rm}\Url}\fi

\bibitem[Auer et~al.(2002)Auer, Cesa-Bianchi, and
  Fischer]{AuerEtAl02FiniteTime}
P.~Auer, N.~Cesa-Bianchi, and P.~Fischer.
\newblock Finite-time analysis of the multiarmed bandit problem.
\newblock \emph{Machine Learning}, 47\penalty0 (2):\penalty0 235--256, 2002.

\bibitem[Bubeck and Cesa-Bianchi(2012)]{BC12}
S.~Bubeck and N.~Cesa-Bianchi.
\newblock Regret analysis of stochastic and nonstochastic multi-armed bandit
  problems.
\newblock \emph{Foundations and Trends in Machine Learning}, 5\penalty0
  (1):\penalty0 1--122, 2012.

\bibitem[Fonteneau-Belmudes(2012)]{Fonteneau11phd}
F.~Fonteneau-Belmudes.
\newblock \emph{{Identification of Dangerous Contingencies for Large Scale
  Power System Security Assessment}}.
\newblock PhD thesis, University of Li\`ege, 2012.

\bibitem[Fonteneau-Belmudes et~al.(2010)Fonteneau-Belmudes, Ernst, Druet,
  Panciatici, and Wehenkel]{FonteneauErnstDruetPanciaticiWehenkel10power}
F.~Fonteneau-Belmudes, D.~Ernst, C.~Druet, P.~Panciatici, and L.~Wehenkel.
\newblock Consequence driven decomposition of large-scale power system security
  analysis.
\newblock In \emph{Proceedings of the 2010 IREP Symposium - Bulk Power Systems
  Dynamics and Control - VIII}, Buzios, Rio de Janeiro, Brazil, August 2010.

\bibitem[Gale and Sampson(1995)]{Gale95GT}
W.A. Gale and G.~Sampson.
\newblock Good-turing frequency estimation without tears.
\newblock \emph{Journal of Quantitative Linguistics}, 2\penalty0 (3):\penalty0
  217--237, 1995.

\bibitem[Garivier and Capp\'e(2011)]{GarivierCappe11KLUCB}
A.~Garivier and O.~Capp\'e.
\newblock The {KL-UCB} algorithm for bounded stochastic bandits and beyond.
\newblock In \emph{Proceedings of the 24rd Annual International Conference on
  Learning Theory}, 2011.

\bibitem[Good(1953)]{Good53GT}
I.J. Good.
\newblock The population frequencies of species and the estimation of
  population parameters.
\newblock \emph{Biometrika}, 40:\penalty0 237--264, 1953.
\newblock ISSN 0006-3444.

\bibitem[McAllester and Ortiz(2003)]{McAllesterOrtiz03concentration}
D.~McAllester and L.~Ortiz.
\newblock Concentration inequalities for the missing mass and for histogram
  rule error.
\newblock \emph{J. Mach. Learn. Res.}, 4:\penalty0 895--911, December 2003.
\newblock ISSN 1532-4435.

\bibitem[McAllester and Schapire(2000)]{McallesterSchapire00GoodTuring}
D.A. McAllester and R.E. Schapire.
\newblock On the convergence rate of {Good-Turing} estimators.
\newblock In \emph{COLT}, pages 1--6, 2000.

\bibitem[McDiarmid(1989)]{Mcdiarmid89boundedDiff}
C.~McDiarmid.
\newblock On the method of bounded differences.
\newblock In \emph{Surveys in combinatorics, 1989 ({N}orwich, 1989)}, volume
  141 of \emph{London Math. Soc. Lecture Note Ser.}, pages 148--188. Cambridge
  Univ. Press, Cambridge, 1989.

\bibitem[Orlitsky et~al.()Orlitsky, Santhanam, and Zhang]{orlitsky2003always}
A.~Orlitsky, N.P. Santhanam, and J.~Zhang.
\newblock Always good {Turing}: Asymptotically optimal probability estimation.
\newblock In \emph{FOCS '03: Proceedings of the 44th Annual IEEE Symposium on
  Foundations of Computer Science}, pages 179+, Washington, DC, USA. IEEE
  Computer Society.
\newblock ISBN 0-7695-2040-5.

\end{thebibliography}

\end{document}